\documentclass[a4,a4wide]{article}

\usepackage{aaai}
\usepackage{times}
\usepackage{helvet}
\usepackage{courier}
\usepackage{amsmath}
\usepackage{amsthm}
\usepackage{color}
\usepackage{comment}
\usepackage{url}
\usepackage{xspace}

\theoremstyle{plain}
 \newtheorem{theorem}{Theorem}
 \newtheorem{corollary}{Corollary}
 
 \newtheorem{proposition}{Proposition}
\theoremstyle{definition}
 \newtheorem{definition}{Definition}
 \newtheorem{example}{Example}

\def\naf{\ensuremath{\raise.17ex\hbox{\ensuremath{\scriptstyle\mathtt{\sim}}}}\xspace}
\def\BP{\ensuremath{\mathcal{B}}\xspace}

\newcounter{generalizedatom}
\newcommand{\generalizedatomlabel}[1]{{\refstepcounter{generalizedatom}\label{#1}\arabic{generalizedatom}}}

\newcounter{program}
\newcommand{\programlabel}[1]{{\refstepcounter{program}\label{#1}\arabic{program}}}

\title{Semantics and Compilation of Answer Set Programming with Generalized Atoms}
\author{Mario Alviano\\
University of Calabria, Italy\\
mario@alviano.com\\
\And
Wolfgang Faber\\
University of Huddersfield, UK\\
wf@wfaber.com}
\begin{document}
\nocopyright
\maketitle

\begin{abstract}
Answer Set Programming (ASP) is logic programming under the stable
model or answer set semantics. During the last decade, this paradigm
has seen several extensions by generalizing the notion of atom used in
these programs. Among these, there are aggregate atoms, HEX atoms,
generalized quantifiers, and abstract constraints. In this paper we
refer to these constructs collectively as generalized atoms. The idea
common to all of these constructs is that their satisfaction depends
on the truth values of a set of (non-generalized) atoms, rather than
the truth value of a single (non-generalized) atom. Motivated by
several examples, we argue that for some of the more intricate
generalized atoms, the previously suggested semantics provide
unintuitive results and provide an alternative semantics, which we
call supportedly stable or SFLP answer sets. We show that it is equivalent to the major
previously proposed semantics for programs with convex generalized
atoms, and that it in general admits more intended models than other
semantics in the presence of non-convex generalized atoms. We show
that the complexity of supportedly stable models is on the second
level of the polynomial hierarchy, similar to previous proposals and
to stable models of disjunctive logic programs. Given these complexity
results, we provide a compilation method that compactly transforms
programs with generalized atoms in disjunctive normal form
to programs without generalized
atoms. Variants are given for the new supportedly stable and the
existing FLP semantics, for which a similar compilation technique has
not been known so far.
\end{abstract}

\section{Introduction}

Answer Set Programming (ASP) is a widely used problem-solving
framework based on logic programming under the stable model
semantics. The basic language relies on Datalog with negation in rule
bodies and possibly disjunction in rule heads. When actually using the
language for representing practical knowledge, it became apparent that
generalizations of the basic language are necessary for
usability. Among the suggested extensions are aggregate atoms (similar
to aggregations in database queries)
\cite{niem-etal-99,niem-simo-2000,dell-etal-2003a,fabe-etal-2008-tplp}
and atoms that rely on external truth valuations
\cite{cali-etal-2007-amai,eite-etal-2004-KR,eite-etal-2005-ijcai}. These
extensions are characterized by the fact that deciding the truth
values of the new kinds of atoms depends on the truth values of a set
of traditional atoms rather than a single traditional atom. We will
refer to such atoms as generalized atoms, which cover also several
other extensions such as abstract constraints, generalized
quantifiers, and HEX atoms.

Concerning semantics for programs containing generalized atoms, there
have been several different suggestions. All of these appear to
coincide for programs that do not contain generalized atoms in
recursive definitions. The two main semantics that emerged as
standards are the PSP semantics defined in
\cite{pelo-2004,pelo-etal-2007-tplp} and \cite{son-pont-2007-tplp},
and the FLP semantics defined in
\cite{fabe-etal-2004-jelia,fabe-etal-2011-aij}. In a recent paper
\cite{alvi-fabe-2013-lpnmr} the relationship between these two
semantics was analyzed in detail; among other, more intricate results,
it was shown that the semantics coincide up to convex generalized
atoms. It was already established earlier that each PSP answer set is
also an FLP answer set, but not vice versa. So for programs containing
non-convex generalized atoms, some FLP answer sets are not PSP answer
sets. In particular, there are programs that have FLP answer sets but
no PSP answer sets.

In this paper, we argue that the FLP semantics is still too
restrictive, and some programs that do not have any FLP answer set
should instead have answer sets. In order to illustrate the point,
consider a coordination game that is remotely inspired by the
prisoners' dilemma. There are two players, each of which has the
option to confess or defect. Let us also assume that both players have
a fixed strategy already, which however still depends on the choice of
the other player as well. In particular, each player will confess
exactly if both players choose the same option, that is, if both
players confess or both defect. The resulting program is $P_1$ in
Example~\ref{ex:program}, where $a$ means that the first player
confesses and $b$ means that the second player confesses. As will be
explained later, the FLP semantics does not assign any answer set to
this program, and therefore also the PSP semantics will not assign any
answer sets to this program. However, this is peculiar, as the
scenario in which both players confess seems like a reasonable one;
indeed, even a simple inflationary operator would result in this
solution.

Looking at the reason why this is not an FLP answer set, we observe
that it has two countermodels that prevent it from being an answer
set: One in which only the first player confesses, and another one in
which only the second player confesses. Both of these countermodels
are models in the classical sense, but they are weak in the sense that
they are not supported, meaning that there is no rule justifying their
truth. This is a situation that does not occur for aggregate-free
programs, which always have supported countermodels. We argue that one
needs to look at supported countermodels, instead of looking at
minimal countermodels. It turns out that doing this yields the same
results not only for aggregate-free programs, but also for programs
containing convex aggregates, which we believe is the reason why this
issue has not been noticed earlier.

In this paper, we define a new semantics along these lines and call it
supportedly stable or SFLP (supportedly FLP) semantics. It provides
answer sets for more programs than FLP and PSP, but is shown to be
equal on convex programs. Analyzing the computational complexity of
the new semantics, we show that it is in the same classes as the FLP
and PSP semantics when considering polynomial-time computable
generalized atoms. It should also be mentioned that the new semantics
has its peculiarities, for instance adding ``tautological'' rules like
$a \leftarrow a$ can change the semantics of the program.

This complexity result directly leads us to the second contribution of
this paper. While it has been known for quite some time that the
complexity of programs with generalized atoms (even without
disjunctions) is equal to the complexity of disjunctive programs, no
compact transformation from programs with generalized atoms to
disjunctive standard programs is known yet. We provide a contribution with 
this respect and show
how to achieve such a compact compilation for both FLP and SFLP
semantics when non-convex aggregates are in disjunctive normal form.
It hinges on the use of disjunction and fresh symbols to
capture satisfaction of a generalized atom.

The remainder of this paper is structured as follows. In the next
section, we present the syntax and FLP semantics for programs with
generalized atoms. After that, we analyze issues with the FLP
semantics and define the SFLP semantics, followed by a section that
proves several useful properties of the new semantics. The subsequent
section then deals with compiling programs with generalized atoms into
generalized-atom-free programs, followed by conclusions.

\section{Syntax and FLP Semantics}

In this section we present the syntax used in this paper and present
the FLP semantics \cite{fabe-etal-2004-jelia,fabe-etal-2011-aij}. To
ease the presentation, we will directly describe a propositional
language here. This can be easily extended to the more usual ASP
notations of programs involving variables, which stand for their
ground versions (that are equivalent to a propositional program).

\subsection{Syntax}

Let \BP be a countable set of \emph{propositional atoms}.

\begin{definition}
  A \emph{generalized atom} $A$ on \BP is a mapping from $2^\BP$ to Boolean truth values.
Each generalized atom $A$ has an associated, finite\footnote{In principle, we could also consider infinite domains, but refrain to do so for simplicity.} domain $D_A \subseteq \BP$, indicating those propositional atoms that are relevant to the generalized atom.
\end{definition}

\begin{example}\label{ex:generalizedatoms}
A generalized atom $A_\generalizedatomlabel{generalizedatom:1}$ modeling a conjunction $a_1, \ldots, a_n$ ($n \geq 0$) of propositional atoms is such that $D_{A_{\ref{generalizedatom:1}}} = \{a_1, \ldots, a_n\}$ and, for every $I \subseteq \BP$, $A_{\ref{generalizedatom:1}}$ maps $I$ to true if and only if $D_{A_{\ref{generalizedatom:1}}} \subseteq I$.

A generalized atom $A_\generalizedatomlabel{generalizedatom:2}$ modeling a conjunction $a_1, \ldots, a_m, \naf a_{m+1}, \ldots, \naf a_n$ ($n \geq m \geq 0$) of literals, where $a_1, \ldots, a_n$ are propositional atoms and $\naf$ denotes \emph{negation as failure}, is such that $D_{A_{\ref{generalizedatom:2}}} = \{a_1, \ldots, a_n\}$ and, for every $I \subseteq \BP$, $A_{\ref{generalizedatom:2}}$ maps $I$ to true if and only if $\{a_1, \ldots, a_m\} \subseteq I$ and $\{a_{m+1}, \ldots, a_n\} \cap I = \emptyset$.

A generalized atom $A_\generalizedatomlabel{generalizedatom:3}$ modeling an aggregate $\mathit{COUNT}(\{a_1, \ldots, a_n\}) \neq k$ ($n \geq k \geq 0$), where $a_1, \ldots, a_n$ are propositional atoms, is such that $D_{A_{\ref{generalizedatom:3}}} = \{a_1, \ldots, a_n\}$ and, for every $I \subseteq \BP$, $A_{\ref{generalizedatom:3}}$ maps $I$ to true if and only if $|D_{A_{\ref{generalizedatom:3}}} \cap I| \neq k$.
\end{example}

In the following, when convenient, we will represent generalized atoms as conjunctions of literals or aggregate atoms.
Subsets of \BP mapped to true by such generalized atoms will be those satisfying the associated conjunction.

\begin{definition}
A general rule $r$ is of the following form:
\begin{equation}\label{eq:rule}
H(r) \leftarrow B(r)
\end{equation}
where $H(r)$ is a disjunction $a_1 \vee \cdots \vee a_n$ ($n \geq 0)$ of propositional atoms in $\BP$ referred to as the head of
$r$, and $B(r)$ is a generalized atom on $\BP$ called the body of $r$.
For convenience, $H(r)$ is sometimes considered a set of propositional atoms.
\end{definition}

A general program $P$ is a set of general rules.

\begin{example}\label{ex:program}
Consider the following rules:
\begin{eqnarray*}
r_1:\quad a & \leftarrow & \mathit{COUNT}(\{a, b\}) \neq 1\\
r_2:\quad b & \leftarrow & \mathit{COUNT}(\{a, b\}) \neq 1
\end{eqnarray*}
The following are general programs:
\begin{eqnarray*}
P_{\programlabel{program:1}} & := & \{r_1; r_2\}\\
P_{\programlabel{program:2}} & := & \{r_1; r_2; a \leftarrow b; b \leftarrow a\}\\
P_{\programlabel{program:3}} & := & \{r_1; r_2; \leftarrow \naf a; \leftarrow \naf b\}\\
P_{\programlabel{program:4}} & := & \{r_1; r_2; a \vee b \leftarrow\}  \\
P_{\programlabel{program:5}} & := & \{r_1; r_2; a \leftarrow \naf b\}  
\end{eqnarray*}
\end{example}

\subsection{FLP Semantics}

An \emph{interpretation} $I$ is a subset of \BP.
$I$ is a \emph{model} for a generalized atom $A$, denoted $I \models A$, if $A$ maps $I$ to true.
Otherwise, if $A$ maps $I$ to false, $I$ is not a model of $A$, denoted $I \not\models A$.
%
$I$ is a model of a rule $r$ of the form (\ref{eq:rule}), denoted $I \models r$, if $H(r) \cap I \neq \emptyset$ whenever $I \models B(r)$.
$I$ is a model of a program $P$, denoted $I \models P$, if $I \models r$ for every rule $r \in P$.

Generalized atoms can be partitioned into two classes according to the following definition.
\begin{definition}[Convex Generalized Atoms]
A generalized atom $A$ is convex if for all triples $I,J,K$ of interpretations such that $I \subset J \subset K$, $I \models A$ and $K \models A$ implies $J \models A$.
\end{definition}
Note that convex generalized atoms are closed under conjunction (but not under disjunction or negation).
A convex program is a general program whose rules have convex bodies.

We now describe a reduct-based semantics, usually referred to as FLP, which has been introduced and analyzed in \cite{fabe-etal-2004-jelia,fabe-etal-2011-aij}.

\begin{definition}[FLP Reduct]
The FLP reduct $P^I$ of a program $P$ with respect to $I$ is defined as the set $\{r \in P \mid I \models B(r)\}$.
\end{definition}

\begin{definition}[FLP Answer Sets]
$I$ is an FLP answer set of $P$ if $I \models P$ and for each $J \subset I$ it holds that $J \not\models P^I$.
Let $FLP(P)$ denote the set of FLP answer sets of $P$.
\end{definition}

\begin{example}\label{ex:models}
Consider the programs from Example~\ref{ex:program}.
The models of $P_{\ref{program:1}}$ are $\{a\}$, $\{b\}$ and $\{a,b\}$, none of which is an FLP answer set.
Indeed, 
\[P_{\ref{program:1}}^{\{a\}} = P_{\ref{program:1}}^{\{b\}} = \emptyset,\]
which have the trivial model $\emptyset$, which is of course a subset of $\{a\}$ and $\{b\}$. On the other hand
\[
P_{\ref{program:1}}^{\{a,b\}} = P_{\ref{program:1}},
\]
and so 
\[
\{a\} \models P_{\ref{program:1}}^{\{a,b\}},
\]
where $\{a\} \subset \{a,b\}$. We will discuss in the next section why this is a questionable situation.

Concerning $P_{\ref{program:2}}$, it has one model, namely $\{a,b\}$, which is also its unique FLP answer set.
Indeed, \[P_{\ref{program:2}}^{\{a,b\}} = P_{\ref{program:2}},\] and hence the only model of $P_{\ref{program:2}}^{\{a,b\}}$ is $\{a,b\}$.

Interpretation $\{a,b\}$ is also the unique model of program $P_{\ref{program:3}}$, which however has no FLP answer set.
Here,
\[P_{\ref{program:3}}^{\{a,b\}} = P_{\ref{program:1}},\] hence similar to $P_{\ref{program:1}}$, \[\{a\} \models P_{\ref{program:3}}^{\{a,b\}}\] and $\{a\} \subset \{a,b\}$.

$P_{\ref{program:4}}$ instead has two FLP answer sets, namely $\{a\}$ and $\{b\}$, and a further model $\{a,b\}$.
In this case, \[P_{\ref{program:4}}^{\{a\}} = \{a \vee b \leftarrow\},\] and no proper subset of $\{a\}$ satisfies it. Also \[P_{\ref{program:4}}^{\{b\}} = \{a \vee b \leftarrow\},\] and no proper subset of $\{b\}$ satisfies it.
Instead, for $\{a,b\}$, we have 
\[P_{\ref{program:4}}^{\{a,b\}}=P_{\ref{program:4}},\]
and hence \[\{a\} \models P_{\ref{program:4}}^{\{a,b\}}\] and $\{a\} \subset \{a,b\}$.

Finally, $P_{\ref{program:5}}$ has tree models, $\{a\}$, $\{b\}$ and $\{a,b\}$, but only one answer set, namely $\{a\}$.
In fact, $P_{\ref{program:5}}^{\{a\}} = \{a \leftarrow \naf b\}$ and $\emptyset$ is not a model of the reduct.
On the other hand, $\emptyset$ is a model of $P_{\ref{program:5}}^{\{b\}} = \emptyset$,
and $\{a\}$ is a model of $P_{\ref{program:5}}^{\{a,b\}} = P_{\ref{program:1}}$.
\end{example}

\section{SFLP Semantics}

As noted in the introduction, the fact that $P_{\ref{program:1}}$ has
no FLP answer sets is striking. If we first assume that both $a$ and
$b$ are false (interpretation $\emptyset$), and then apply a
generalization of the well-known one-step derivability operator, we
obtain truth of both $a$ and $b$ (interpretation $\{a,b\}$). Applying
this operator once more again yields the same interpretation, a
fix-point. $\{a,b\}$ is also a supported model, that is, for all true
atoms there exists a rule in which this atom is the only true head
atom, and in which the body is true. 

It is instructive to examine why this seemingly robust model is not an
FLP answer set. Its reduct is equal to the original program,
$P_{\ref{program:1}}^{\{a,b\}} = P_{\ref{program:1}}$. There are
therefore two models of $P_{\ref{program:1}}$, $\{a\}$ and $\{b\}$,
that are subsets of $\{a,b\}$ and therefore inhibit $\{a,b\}$ from
being an FLP answer set. The problem is that, contrary to $\{a,b\}$,
these two models are rather weak, in the sense that they are not
supported. Indeed, when considering $\{a\}$, there is no rule in
$P_{\ref{program:1}}$ such that $a$ is the only true atom in the rule
head and the body is true in $\{a\}$: The only available rule with $a$
in the head has a false body. The situation for $\{b\}$ is symmetric.

It is somewhat counter-intuitive that a model like $\{a,b\}$ should be
inhibited by two weak models like $\{a\}$ and $\{b\}$. Indeed, this is
a situation that normally does not occur in ASP. For programs that do
not contain generalized atoms, whenever one finds a $J \subseteq I$
such that $J \models P^I$ there is for sure also a $K \subseteq I$
such that $K \models P^I$ and $K$ is supported. Indeed, we will show
in the following section that this is the case also for programs
containing only convex generalized atoms. Our feeling is that since such
a situation does not happen for a very wide set of programs, it has
been overlooked so far.

We will now attempt to repair this kind of anomaly by stipulating that
one should only consider supported models for finding inhibitors of
answer sets. In other words, one does not need to worry about
unsupported models of the reduct, even if they are subsets of the
candidate. Let us first define supported models explicitly.

\begin{definition}[Supportedness]
A model $I$ of a program $P$ is supported if for each $a \in I$ there is a rule $r \in P$ such that $I \cap H(r) = \{a\}$ and $I \models B(r)$.
In this case we will write $I \models_s P$.  
\end{definition}

\begin{example}\label{ex:supp-models}
Continuing Example~\ref{ex:models}, programs $P_{\ref{program:1}}$, $P_{\ref{program:2}}$, and $P_{\ref{program:3}}$ have one supported model, namely $\{a,b\}$.
The model $\{a\}$ of $P_{\ref{program:1}}$ is not supported because the body of the the rule with $a$ in the head has a false body with respect to $\{a\}$. For a symmetric argument, model $\{b\}$ of $P_{\ref{program:1}}$ is not supported either.
The supported models of $P_{\ref{program:4}}$, instead, are $\{a\}$, $\{b\}$, and $\{a,b\}$, so all models of the program are supported. Note that both models $\{a\}$ and $\{b\}$ have the disjunctive rule as the only supporting rule for the respective single true atom, while for $\{a,b\}$, the two rules with generalized atoms serve as supporting rules for $a$ and $b$.
Finally, the supported models of $P_{\ref{program:5}}$ are $\{a\}$ and $\{a,b\}$.
\end{example}

We are now ready to formally introduce the new semantics. In this paper we will normally refer to it as SFLP answer sets or SFLP semantics, but also call it supportedly stable models occasionally.

\begin{definition}[SFLP Answer Sets]\label{def:SFLP}
$I$ is a supportedly FLP answer set (or SFLP answer set, or supportedly stable model) of $P$ if $I \models_s P$ and for each $J \subset I$ it holds that $J \not\models_s P^I$.
Let $SFLP(P)$ denote the set of SFLP answer sets of $P$.
\end{definition}

\begin{example}\label{ex:SFLP}
Consider again the programs from Example~\ref{ex:program}.
Recall that $P_{\ref{program:1}}$ has only one supported model, namely $\{a,b\}$, and
\[
P_{\ref{program:1}}^{\{a,b\}} = P_{\ref{program:1}},
\]
but 
\[
\begin{array}{l}
\emptyset \not\models_s P_{\ref{program:1}}^{\{a,b\}},\\
\{a\} \not\models_s P_{\ref{program:1}}^{\{a,b\}},\\
\{b\} \not\models_s P_{\ref{program:1}}^{\{a,b\}},  
\end{array}
\]
therefore no proper subset of $\{a,b\}$ is a supported model, hence it is an SFLP answer set.

Concerning $P_{\ref{program:2}}$, it has one model, namely $\{a,b\}$, which is supported and also its unique SFLP answer set.
Indeed, recall that \[P_{\ref{program:2}}^{\{a,b\}} = P_{\ref{program:2}},\] and hence no proper subset of $\{a,b\}$ can be a model (let alone a supported model) of $P_{\ref{program:2}}^{\{a,b\}}$.

Interpretation $\{a,b\}$ is the unique model of program $P_{\ref{program:3}}$, which is supported and also its SFLP answer set.
In fact
\[P_{\ref{program:3}}^{\{a,b\}} = P_{\ref{program:1}}.\]

$P_{\ref{program:4}}$ has two SFLP answer sets, namely $\{a\}$ and $\{b\}$.
In this case, recall \[P_{\ref{program:4}}^{\{a\}} = \{a \vee b \leftarrow\},\] and no proper subset of $\{a\}$ satisfies it. Also \[P_{\ref{program:4}}^{\{b\}} = \{a \vee b \leftarrow\},\] and no proper subset of $\{b\}$ satisfies it. Instead, for $\{a,b\}$, we have 
\[P_{\ref{program:4}}^{\{a,b\}}=P_{\ref{program:4}},\]
hence since 
\[
\begin{array}{l}
\{a\} \models_s P_{\ref{program:4}}^{\{a,b\}},\\
\{b\} \models_s P_{\ref{program:4}}^{\{a,b\}},  
\end{array}
\]
we obtain that $\{a,b\}$ is not an SFLP answer set.

Finally, $P_{\ref{program:5}}$ has two SFLP answer sets, namely $\{a\}$ and $\{a,b\}$.
In fact, $P_{\ref{program:5}}^{\{a\}} = \{a \leftarrow \naf b\}$ and $P_{\ref{program:5}}^{\{a,b\}} = P_{\ref{program:1}}$.

The programs, models, FLP answer sets, supported models, and SFLP answer sets are summarized in Table~\ref{table:models}.
\end{example}

\begin{table*}
 \caption{(Supported) models and (S)FLP answer sets of programs in Example~\ref{ex:program}, where $A := \mathit{COUNT}(\{a, b\}) \neq 1$.}
\label{table:models}
 
 \medskip
 \begin{center}
   
 \begin{tabular}{llllll}
  & \bf Rules & \bf Models & \bf FLP & \bf Supported Models & \bf SFLP \\
  \hline
  \hline
  $P_1$ & $a \leftarrow A$ \quad $b \leftarrow A$ & $\{a\}$, $\{b\}$, $\{a,b\}$ & --- & $\{a,b\}$ & $\{a,b\}$ \\
  \hline
  $P_2$ & $a \leftarrow A$ \quad $b \leftarrow A$ \quad $a \leftarrow b$ \quad $b \leftarrow a$ & $\{a,b\}$ & $\{a,b\}$ & $\{a,b\}$ & $\{a,b\}$ \\
  \hline
  $P_3$ & $a \leftarrow A$ \quad $b \leftarrow A$ \quad $\leftarrow \naf a$ \quad $\leftarrow \naf b$ & $\{a,b\}$ & --- & $\{a,b\}$ & $\{a,b\}$ \\
  \hline
  $P_4$ & $a \leftarrow A$ \quad $b \leftarrow A$ \quad $a \vee b \leftarrow$ & $\{a\}$, $\{b\}$, $\{a,b\}$ & $\{a\}$, $\{b\}$ & $\{a\}$, $\{b\}$, $\{a,b\}$ & $\{a\}$, $\{b\}$ \\
  \hline
  $P_5$ & $a \leftarrow A$ \quad $b \leftarrow A$ \quad $a \leftarrow \naf b$ & $\{a\}$, $\{b\}$, $\{a,b\}$ & $\{a\}$ & $\{a\}$, $\{a,b\}$ & $\{a\}$, $\{a,b\}$ \\
 \end{tabular}
 \end{center}
\end{table*}

An alternative, useful characterization of SFLP answer sets can be
given in terms of Clark's completion \cite{clar-78}.  In fact, it is
well-known that supported models of a program are precisely the models
of its completion. We define this notion in a somewhat non-standard
way, making use of the concept of generalized atom.

Next, we first define the completion of a propositional atom $a$ with respect to a general program $P$ as a generalized atom encoding the supportedness condition for $a$.

\begin{definition}
The completion of a propositional atom $a \in \BP$ with respect to a general program $P$ is a generalized atom $comp(a,P)$ mapping to true any interpretation $I$ containing $a$ and such that
there is no rule $r \in P$ for which 
$I \models B(r)$ and
$I \cap H(r) = \{a\}$.
\end{definition}

These generalized atoms are then used to effectively define a program whose models are the supported model of $P$.

\begin{definition}
The completion of a general program $P$ is a general program $comp(P)$ extending $P$ with a rule
$$\leftarrow comp(a,P)$$
for each propositional atom $a$ occurring in $P$.  
\end{definition}

\begin{example}
Consider again programs from Example~\ref{ex:program}.
Program $comp(P_{\ref{program:1}})$ extends $P_{\ref{program:1}}$ with the following rules:
\begin{eqnarray*}
 & \leftarrow & a,\ \mathit{COUNT}(\{a, b\}) = 1\\
 & \leftarrow & b,\ \mathit{COUNT}(\{a, b\}) = 1
\end{eqnarray*}
Program $comp(P_{\ref{program:2}})$ extends $P_{\ref{program:2}}$ with the following rules:
\begin{eqnarray*}
 & \leftarrow & a,\ \mathit{COUNT}(\{a, b\}) = 1,\ \naf b\\
 & \leftarrow & b,\ \mathit{COUNT}(\{a, b\}) = 1,\ \naf a
\end{eqnarray*}
Program $comp(P_{\ref{program:3}})$ is equal to $comp(P_{\ref{program:1}})$, and program $comp(P_{\ref{program:4}})$ extends $P_{\ref{program:4}}$ with the following rules:
\begin{eqnarray*}
 & \leftarrow & a,\ \mathit{COUNT}(\{a, b\}) = 1,\ b\\
 & \leftarrow & b,\ \mathit{COUNT}(\{a, b\}) = 1,\ a
\end{eqnarray*}
Program $comp(P_{\ref{program:5}})$ instead extends $P_{\ref{program:5}}$ with the following rules:
\begin{eqnarray*}
 & \leftarrow & a,\ \mathit{COUNT}(\{a, b\}) = 1,\ b\\
 & \leftarrow & b,\ \mathit{COUNT}(\{a, b\}) = 1
\end{eqnarray*}
The only model of $comp(P_{\ref{program:1}})$, $comp(P_{\ref{program:2}})$, and $comp(P_{\ref{program:3}})$ is $\{a,b\}$.
The models of $comp(P_{\ref{program:4}})$ and $comp(P_{\ref{program:5}})$ instead are $\{a\}$, $\{b\}$, and $\{a,b\}$.
\end{example}

\begin{proposition}
Let $P$ be a general program and $I$ an interpretation.  $I \models_s P$ iff $I \models comp(P)$.
\end{proposition}

This characterization (which follows directly from \cite{clar-78}) provides us with a means for implementation that relies only on model checks, rather than supportedness checks.

\begin{proposition}
  Let $P$ be a general program and $I$ an interpretation. 
$I$ is a supportedly FLP answer set of $P$ if $I \models comp(P)$ and for each $J \subset I$ it holds that $J \not\models comp(P^I)$.
\end{proposition}

\section{Properties}

The new semantics has a number of interesting properties that we report in this section.
First of all, it is an extension of the FLP semantics, in the sense that each FLP answer set is also an SFLP answer set.

\begin{theorem}\label{thm:extension}
Let $P$ be a general program.
$FLP(P) \subseteq SFLP(P)$.
\end{theorem}
\begin{proof}
Let $I$ be an FLP answer set of $P$.
Hence, each $J \subset I$ is such that $J \not\models P^I$.
Thus, we can conclude that $J \not\models_s P^I$ for any $J \subset I$. Therefore, $I$ is a SFLP answer set of $P$.
\end{proof}

The inclusion is strict in general.
In fact, $P_{\ref{program:1}}$ is a simple program for which the two semantics disagree (see Examples~\ref{ex:program}--\ref{ex:SFLP} and Table~\ref{table:models}).
On the other hand, the two semantics are equivalent for a large class of programs, as shown below.

\begin{theorem}\label{thm:convex-equivalence}
If $P$ is a convex program then $FLP(P) = SFLP(P)$.
\end{theorem}
\begin{proof}
$FLP(P) \subseteq SFLP(P)$ holds by Theorem~\ref{thm:extension}.
For the other direction, consider an interpretation $I$ not being an FLP answer set of $P$.
Hence, there is $J \subset I$ such that $J \models P^I$.
We also assume that $J$ is a subset-minimal model of $P^I$, that is, there is no $K \subset J$ such that $K \models P^I$.
We shall show that $J \models_s P^I$. 
To this end, suppose by contradiction that there is $a \in J$ such that for each $r \in P^I$ either $J \not\models B(r)$ or $J \cap H(r) \neq \{a\}$.
Consider $J \setminus \{a\}$ and a rule $r \in P^I$ such that $J \setminus \{a\} \models B(r)$.
Since $r \in P^I$, $I \models B(r)$, and thus $J \models B(r)$ because $B(r)$ is convex.
Therefore, $J \cap H(r) \neq \{a\}$.
Moreover, $J \cap H(r) \neq \emptyset$ because $J \models P^I$ by assumption.
Hence, $(J \setminus \{a\}) \cap H(r) \neq \emptyset$, and therefore $J \setminus \{a\} \models P^I$.
This contradicts the assumption that $J$ is a subset-minimal model of $P^I$.
\end{proof}

We will now focus on computational complexity. We consider here the
problem of determining whether an SFLP answer set exists. We note that
the only difference to the FLP semantics is in the stability
check. For FLP, subsets need to be checked for being a model, for
SFLP, subsets need to be checked for being a supported
model. Intuitively, one would not expect that this difference can
account for a complexity jump, which is confirmed by the next result.

\begin{theorem}\label{thm:complexity}
Let $P$ be a general program whose generalized atoms are polynomial-time computable functions.
Checking whether $SFLP(P) \neq \emptyset$ is in $\Sigma^P_2$ in general;
it is $\Sigma^P_2$-hard already in the disjunction-free case if at least one form of non-convex generalized atom is permitted.
The problem is $NP$-complete if $P$ is disjunction-free and convex.
\end{theorem}
\begin{proof}
For the membership in $\Sigma^P_2$ one can guess an interpretation $I$ and check that there is no $J \subset I$ such that $J \models_s P$.
The check can be performed by a $coNP$ oracle.

To prove $\Sigma^P_2$-hardness we note that extending a general program $P$ by rules $a \leftarrow a$ for every propositional atom occurring in $P$ is enough to guarantee that all models of any reduct of $P$ are supported.
We thus refer to the construction and proof by \cite{alvi-fabe-2013-lpnmr}.

If $P$ is disjunction-free and convex then $SFLP(P) = FLP(P)$ by Theorem~\ref{thm:convex-equivalence}.
Hence, $NP$-completeness follows from results in \cite{liu-trus-2006-jair}.
\end{proof}

We would like to point out that the above proof also illustrates a
peculiar feature of SFLP answer sets, which it shares with the
supported model semantics: the semantics is sensitive to tautological
rules like $a \leftarrow a$, as their addition can turn non-SFLP
answer sets into SFLP answer sets.

\section{Compilation}

The introduction of generalized atoms in logic programs does not increase the computational complexity of checking FLP as well as SFLP answer set existence, as long as one is allowed to use disjunctive rule heads.
However, so far no compilation method that compactly transforms general programs to logic programs without generalized atoms has been presented for the FLP semantics.
In the following we provide such a compilation for non-convex aggregates in disjunctive normal form.
The compilation is also extended for the new SFLP semantics.
We point out that such compilations are not necessarily intended to provide efficient methods for computing answer sets of general programs.
Their purpose is instead to provide insights that may lead to obtain such methods in the future. 

In this section we only consider generalized atoms in disjunctive normal form, that is, a generalized atom $A$ will be associated with an equivalent propositional formula of the following form:
\begin{equation}\label{eq:dnf}
\bigvee_{i=1}^k a_{i_1} \wedge \ldots \wedge a_{i_m} \wedge \naf a_{i_{m+1}} \wedge \ldots \wedge \naf a_{i_n}
\end{equation}
where $k \geq 1$, $i_{n} \geq i_{m} \geq 0$ and $a_{i_1}, \ldots, a_{i_n}$ are propositional atoms for $i = 1, \ldots, k$.
We will also assume that the programs to be transformed have atomic heads. 
To generalize our compilations to cover disjunctive general rules is a problem to be addressed in future work.

Let $P$ be a program.
In our construction we will use the following \emph{fresh} propositional atoms, i.e., propositional atoms not occurring in $P$:
$A^T$ for each generalized atom $A$;
$A^{F_i}$ for each generalized atom $A$ and integer $i \geq 0$.
For a generalized atom $A$ of the form (\ref{eq:dnf}) and integer $i = 1, \ldots, k$, let $tr(A, i)$ denote the following rule:
\begin{equation}\label{eq:tr}
A^T \vee a_{i_{m+1}} \vee \cdots \vee a_{i_n} \leftarrow a_{i_1}, \ldots, a_{i_m}, \naf A^{F_0}.
\end{equation}
Moreover, let $fls(A,i,j)$ denote
\begin{equation}\label{eq:fls-positive}
A^{F_i} \leftarrow \naf a_{i_j}, \naf A^T
\end{equation}
for $j = i_1, \ldots, i_m$, and
\begin{equation}\label{eq:fls-negative}
A^{F_i} \leftarrow a_{i_j}, \naf A^T
\end{equation}
for $j = i_{m+1}, \ldots, i_n$.
Abusing of notation, let $fls(A)$ denote the following rule:
\begin{equation}\label{eq:fls}
A^{F_0} \leftarrow A^{F_1}, \ldots, A^{F_k}, \naf A^T.
\end{equation}
Intuitively, rule $tr(A,i)$ forces truth of $A^T$ whenever the $i$-th disjunct of $A$ is true.
Similarly, rule $fls(A,i,j)$ forces truth of $A^{F_i}$ whenever the $i$-th disjunct of $A$ is false due to atom $a_{i_j}$;
if all disjuncts of $A$ are false, rule $fls(A)$ forces truth of $A^{F_0}$ to model that $A$ is actually false.
Note that atoms occurring in negative literals of the $i$-th disjunct of $A$ have been moved in the head of $tr(A,i)$.
In this way, the information encoded by $tr(A,i)$ is preserved in the reduct with respect to an interpretation $I$ whenever the $i$-th disjunct of $A$ is true with respect to a subset of $I$, not necessarily $I$ itself.

The rewriting of $A$, denoted $rew(A)$, is the following set of rules:
\begin{equation}
\begin{split}
\{tr(A,i) \mid i = 1, \ldots, k\} \cup \{fls(A)\} \cup {} & \\
\{fls(A,i,j) \mid i = 1, \ldots, k \wedge j = 1, \ldots, n\} &
\end{split}
\end{equation}
The rewriting of $P$, denoted $rew(P)$, is obtained from $P$ by replacing each generalized atom $A$ by $A^T$.
The FLP-rewriting of $P$, denoted $rew^{FLP}(P)$, is obtained from $rew(P)$ by adding rules in $rew(A)$ for each generalized atom $A$ occurring in $P$.
The SFLP-rewriting of $P$, denoted $rew^{SFLP}(P)$, is obtained from $rew^{FLP}(P)$ by adding a rule $supp(a)$ of the form
\begin{equation}\label{eq:completion}
A_1^T \vee \cdots \vee A_n^T \leftarrow a
\end{equation}
for each propositional atom $a$ occurring in $P$, where $a \leftarrow A_i$ ($i = 1, \ldots, n$) are the rules of $P$ having head $a$.

\begin{example}
Let $A$ be the generalized atom in Example~\ref{ex:program}.
Its disjunctive normal form is $\naf a \wedge \naf b \vee a \wedge b$.
Rules $r_1$ and $r_2$ are then $a \leftarrow A$ and $b \leftarrow A$.
Program $rew^{FLP}(P_{\ref{program:1}})$ is
$$
\begin{array}{rrcl}
rew(\{r_1\}): & a & \leftarrow & A^T\\
rew(\{r_2\}): & b & \leftarrow & A^T\\
tr(A,1): & A^T \vee a \vee b & \leftarrow & \naf A^{F_0}\\
tr(A,2): & A^T & \leftarrow & a, b, \naf A^{F_0}\\
fls(A,1,1): & A^{F_1} & \leftarrow & a, \naf A^T\\
fls(A,1,2): & A^{F_1} & \leftarrow & b, \naf A^T\\
fls(A,2,1): & A^{F_2} & \leftarrow & \naf a, \naf A^T\\
fls(A,2,2): & A^{F_2} & \leftarrow & \naf b, \naf A^T\\
fls(A): & A^{F_0} & \leftarrow & A^{F_1}, A^{F_2}, \naf A^T\\
\end{array}
$$
One can check that $rew^{FLP}(P_{\ref{program:1}})$ has no answer set. In particular,
$\{a,b,A^T\}$ is not an answer set of $rew^{FLP}(P_{\ref{program:1}})$. Its FLP reduct consists of the first four rules
$$
\begin{array}{rcl}
a & \leftarrow & A^T\\
b & \leftarrow & A^T\\
A^T \vee a \vee b & \leftarrow & \naf A^{F_0}\\
A^T & \leftarrow & a, b, \naf A^{F_0}
\end{array}
$$
and both $\{a\}$ and $\{b\}$ are minimal models of the reduct. On the other hand, neither $\{a\}$ nor $\{b\}$ are models of the original program, and so also not answer sets.

Program $rew^{SFLP}(P_{\ref{program:1}})$ extends $rew^{FLP}(P_{\ref{program:1}})$ with the following rules:
$$
\begin{array}{rrcl}
supp(a): & A^T & \leftarrow & a\\
supp(b): & A^T & \leftarrow & b
\end{array}
$$ 
The program $rew^{SFLP}(P_{\ref{program:1}})$ has one answer set:
\[
\{a,b,A^T\}.\]
 In contrast to $rew^{FLP}(P_{\ref{program:1}})$ its FLP
reduct now consists of the first four rules of
$rew^{FLP}(P_{\ref{program:1}})$ plus the two additional rules:
$$
\begin{array}{rcl}
a & \leftarrow & A^T\\
b & \leftarrow & A^T\\
A^T \vee a \vee b & \leftarrow & \naf A^{F_0}\\
A^T & \leftarrow & a, b, \naf A^{F_0}\\
A^T & \leftarrow & a\\
A^T & \leftarrow & b
\end{array}
$$
 These
two additional rules impede $\{a\}$ and $\{b\}$ to be models, and
indeed only $\{a,b,A^T\}$ is a model of the reduct.

Program $rew^{FLP}(P_{\ref{program:2}})$ is $rew^{FLP}(P_{\ref{program:1}}) \cup \{a \leftarrow b; b \leftarrow a\}$.
(To simplify the presentation, bodies equivalent to atomic literals are not rewritten.)

In this case, \[
\{a,b,A^T\}\] is its only answer set. Different to $rew^{FLP}(P_{\ref{program:2}})$, the additional rules will be present in the reduct for $\{a,b,A^T\}$:
$$
\begin{array}{rcl}
a & \leftarrow & A^T\\
b & \leftarrow & A^T\\
A^T \vee a \vee b & \leftarrow & \naf A^{F_0}\\
A^T & \leftarrow & a, b, \naf A^{F_0}\\
a & \leftarrow & b\\
b & \leftarrow & a
\end{array}
$$
Thus the reduct models $\{a\}$ and $\{b\}$ are avoided.

Program $rew^{SFLP}(P_{\ref{program:2}})$ extends $rew^{FLP}(P_{\ref{program:2}})$ with
$$
\begin{array}{rrcl}
supp(a)': & A^T \vee b & \leftarrow & a\\
supp(b)': & A^T \vee a & \leftarrow & b
\end{array}
$$
It is easy to see that these additional rules do not alter answer sets, so also $rew^{SFLP}(P_{\ref{program:2}})$ has a single answer set $\{a,b,A^T\}$.

Program $rew^{FLP}(P_{\ref{program:3}})$ is $rew^{FLP}(P_{\ref{program:1}}) \cup \{\leftarrow \naf a; \leftarrow \naf b\}$.
This program has no answer sets for the same reason as $rew^{FLP}(P_{\ref{program:1}})$. Indeed, the two additional rules are not in the reduct for $\{a,b,A^T\}$, and so $\{a\}$ and $\{b\}$ are again minimal models.

Program $rew^{SFLP}(P_{\ref{program:3}})$ is $rew^{SFLP}(P_{\ref{program:1}}) \cup \{\leftarrow \naf a; \leftarrow \naf b\}$. For the same reason as for $rew^{SFLP}(P_{\ref{program:1}})$, this program has exactly one answer set: 
\[
\{a,b,A^T\}.\]
 The two new rules disappear in the reduct, but the rules present in $rew^{SFLP}(P_{\ref{program:1}})$ but not in $rew^{FLP}(P_{\ref{program:1}})$ do not allow models $\{a\}$ and $\{b\}$.
    
Program $P_{\ref{program:4}}$ contains a disjunctive rule and is thus not in the domain of $rew^{FLP}$ and $rew^{SFLP}$ described here.
\end{example}

In the examples provided so far, it can be checked that answer sets are preserved by our transformations if auxiliary symbols are ignored.
In the remainder of this section we will formalize this intuition.

\begin{definition}
The expansion of an interpretation $I$ for a program $P$, denoted $exp(I)$, is the following interpretation:
\begin{equation}
\begin{split}
I & \cup \{A^T \mid A^T \mbox{ occurs in } rew(P),\ I \models A\} \\
{} & \cup \{A^{F_i} \mid A^{F_i} \mbox{ occurs in } rew(P),\ I \not\models A\}.
\end{split}
\end{equation}
The contraction of an interpretation $I$ to the symbols of $P$, denoted $I|_P$, is the following interpretation:
\begin{equation}
I \cap \{a \in \BP \mid a \mbox{ occurs in } P\}.
\end{equation}
\end{definition}

Below, we show that expansions and contractions define bijections between the answer sets of a program and those of the corresponding compilations.
In the claim we consider only FLP answer sets of the rewritten program because it is convex, and thus its FLP and SFLP answer sets coincide by Theorem~\ref{thm:convex-equivalence}.

\begin{theorem}\label{thm:compilation}
Let $P$ be a program, and ${\mathcal F} \in \{FLP, SFLP\}$.
\begin{enumerate}
\item
If $I \in {\mathcal F}(P)$ then $exp(I) \in FLP(rew^{\mathcal F}(P))$.

\item 
If $I \in FLP(rew^{\mathcal F}(P))$ then $I|_P \in {\mathcal F}(P)$.
\end{enumerate}
\end{theorem}

\begin{proof}[Proof (item 1)]
Let $I$ be an $\mathcal F$ answer set of $P$.
Hence, $I \models_s P$ (see Definition~\ref{def:SFLP} and Theorem~\ref{thm:extension}). 
Since each generalized atom $A$ occurring in $P$ is replaced by $A^T$ in $rew(P)$, and $A^T \in exp(I)$ if and only if $I \models A$, we have $I \models rew(P)$.
Consider rules in $rew(A)$ for some generalized atom $A$ of the form (\ref{eq:dnf}) occurring in $P$, and note that either $A^T \in exp(I)$ or $A^{F_0}, \ldots, A^{F_k} \in exp(I)$.
In both cases, all rules in $rew(A)$ are satisfied by $exp(I)$.
Hence, $exp(I) \models rew^{FLP}(P)$.
%
Consider a rule $supp(a)$ of the form (\ref{eq:completion}) such that $a \in I$.
Since $I \models_s P$, there is $i \in \{1, \ldots, n\}$ such that $I \models A_i$.
Thus, $A_i^T \in exp(I)$, and therefore $exp(I) \models supp(a)$.
We can conclude $exp(I) \models rew^{SFLP}(P)$.

Let $J \subseteq exp(I)$ be such that $J \models rew^{\mathcal F}(P)^{exp(I)}$.
We first show that $J|_P = I$.
Consider a rule $a \leftarrow A$ in $P^I$ such that $I \models A$ and $J|_P \models A$, where $A$ is of the form (\ref{eq:dnf}).
Hence, there is $i \in \{1, \ldots, k\}$ such that \[J|_P \models a_{i_1} \wedge \ldots \wedge a_{i_m} \wedge \naf a_{i_{m+1}} \wedge \ldots \wedge \naf a_{i_n}.\]
Therefore, $A^T \in J$ because $tr(A,i) \in rew^{\mathcal F}(P)^{exp(I)}$, and consequently $a \in J$ because of rule $a \leftarrow A^T$ in $rew^{\mathcal F}(P)^{exp(I)}$.
We thus conclude $J|_P \models P^I$.
For ${\mathcal F} = FLP$, this already proves $J|_P = I$.
For ${\mathcal F} = SFLP$, let $X \subseteq J|_P$ be the atoms without support, i.e., $X$ is a subset-maximal set such that $a \in X$ implies $J|_P \setminus X \not\models A$ for each rule $a \leftarrow A$ in $P^I$.
Hence, $J|_P \setminus X \models_s P^I$.
It follows that $J|_P \setminus X = I$, i.e., $X = \emptyset$ and $J|_P = I$.

We can now show that $J = exp(I)$.
Let $A$ be a generalized atom of the form (\ref{eq:dnf}).
If $J|_P \models A$ there is $i \in \{1, \ldots, k\}$ such that \[J|_P \models a_{i_1} \wedge \ldots \wedge a_{i_m} \wedge \naf a_{i_{m+1}} \wedge \ldots \wedge \naf a_{i_n},\] and thus $A^T \in J$ because $tr(A,i) \in rew^{\mathcal F}(P)^{exp(I)}$ and $J \models rew^{\mathcal F}(P)^{exp(I)}$.
Otherwise, if $J|_P \not\models A$ then for all $i \in \{1, \ldots, k\}$ there is either $j \in \{1, \ldots, m\}$ such that $a_{i_j} \notin J|_P$, or $j \in \{m+1, \ldots, n\}$ such that $a_{i_j} \in J|_P$.
Hence, $A^{F_i} \in J$ because $J \models fls(A,i,j)$, and thus $A^{F_0} \in J$ because $J \models fls(A)$.
\end{proof}

\begin{proof}[Proof (item 2)]
Let $I$ be an FLP answer set of $rew^{\mathcal F}(P)$.
Let $A$ be a generalized atom $A$ of the form (\ref{eq:dnf}) occurring in $P$.
We prove the following statements:
\begin{eqnarray}
|I \cap \{A^T,A^{F_i}\}| \leq 1 \mbox{ holds for } i = 1, \ldots, k \label{eq:at-most-one}\\
A^T \in I \mbox{ if and only if } I|_P \models A \label{eq:iff}\\
|I \cap \{A^T,A^{F_i}\}| = 1 \mbox{ holds for } i = 1, \ldots, k \label{eq:one}
\end{eqnarray}

To prove (\ref{eq:at-most-one}), 
define set $X$ as a maximal subset satisfying the following requirements:
If $\{A^T,A^{F_i}\} \subseteq I$ (for some $i \in \{1,\ldots,k\}$) then $\{A^T,A^{F_0},\ldots,A^{F_k}\} \subseteq X$;
if an atom $a$ is not supported by $I \setminus X$ in $rew^{FLP}(P)^I$ then $a \in X$.
We have $I \setminus X \models rew^{\mathcal F}(P)^I$, from which we conclude $X = \emptyset$.

Consider (\ref{eq:iff}).
If $A^T \in I$ then by (\ref{eq:at-most-one}) no $A^{F_i}$ belongs to $I$.
Recall that FLP answer sets are supported models, i.e., $I \models_s rew^{\mathcal F}(P)$.
Thus, for ${\mathcal F} = FLP$, there is $i \in \{1,\ldots,k\}$ such that $I \models B(tr(A,i))$ and $I \cap H(tr(A,i)) = \{A^T\}$.
Therefore, $I|_P \models A$.
For ${\mathcal F} = SFLP$, we just note that if $A^T$ is supported only by a rule of the form (\ref{eq:completion}), then atom $a$ is only supported by a rule $a \leftarrow A^T$ in $rew^{\mathcal F}(P)$.
$I \setminus \{a, A^T\}$ would be a model of $rew^{\mathcal F}(P)^I$ in this case, then contradicting $I \in FLP(rew^{\mathcal F}(P))$.
Now consider the right-to-left direction.
If $I|_P \models A$ then there is $i \in \{1, \ldots, k\}$ such that $I|_P \models a_{i_1} \wedge \ldots \wedge a_{i_m} \wedge \naf a_{i_{m+1}} \wedge \ldots \wedge \naf a_{i_n}$, and thus $A^{F_i} \notin I$ (see Equations~\ref{eq:fls-positive}--\ref{eq:fls-negative}).
Hence, $A^{F_0} \notin I$ (see Equation~\ref{eq:fls}).
From rule $tr(A,i)$ (see Equation~\ref{eq:tr}) we have $A^T \in I$.

Concerning (\ref{eq:one}), because of (\ref{eq:at-most-one}) and (\ref{eq:iff}), we have just to show that $A^{F_0},\ldots,A^{F_k} \in I$ whenever $I|_P \not\models A$.
In fact, in this case $A^T \notin I$ by (\ref{eq:iff}), and for each $i \in \{1,\ldots,k\}$ there is either $j \in \{1, \ldots, m\}$ such that $a_{i_j} \notin I|_P$, or $j \in \{m+1, \ldots, n\}$ such that $a_{i_j} \in I|_P$.
Hence, $A^{F_i} \in I$ because of rules $fls(r,i,j)$ and $fls(r)$.

We can now prove the main claim.
We start by showing that $I|_P \models P$.
Indeed, for a rule $a \leftarrow A$ in $P$ such that $I|_P \models A$, $rew(P)$ contains a rule $a \leftarrow A^T$.
Moreover, $A^T \in I$ by (\ref{eq:iff}), and thus $a \in I$.
If ${\mathcal F} = SFLP$, then for each $a \in I$ we have $I \models supp(a)$, where $supp(a)$ is of the form (\ref{eq:completion}). Hence, there is $i \in {1,...,n}$ such that $A_i^T \in I$. Therefore, (\ref{eq:iff}) implies $I|_P \models A_i$, that is, $a$ is supported by $I|_P$ in $P$. We can thus conclude that $I_P \models_s P$.

To complete the proof, for ${\mathcal F} = FLP$ we consider $X \subseteq I|_P$ such that $I|_P \setminus X \models P^{I|_P}$, while for ${\mathcal F} = SFLP$ we consider $X \subseteq I|_P$ such that $I|_P \setminus X \models_s P^{I|_P}$.
Let $J$ be the interpretation obtained from $I \setminus X$ by removing all atom $A^T$ such that $I|_P \setminus X \not\models A$.
We shall show that $J \models rew^{\mathcal F}(P)^I$, from which we conclude $X = \emptyset$.
Consider a rule of the form $a \leftarrow A^T$ in $rew^{\mathcal F}(P)^I$ such that $A^T \in J$.
Hence, $I|_P \setminus X \models A$ by construction of $J$.
Since $a \leftarrow A$ is a rule in $P^{I|_P}$, we conclude $a \in I|_P \setminus X$ and thus $a \in J$.
Consider now a rule $tr(A,i)$ in $rew^{\mathcal F}(P)^I$ such that $J \models B(tr(A,i))$ and $A^T \notin J$.
Hence, $I|_P \setminus X \not\models A$ by construction of $J$, which means that there is either $j \in \{1, \ldots, m\}$ such that $a_{i_j} \notin I|_P \setminus X$, or $j \in \{m+1, \ldots, n\}$ such that $a_{i_j} \in I|_P \setminus X$.
We conclude that $J \models tr(A,i)$.
Rules $fls(A,i,j)$ and $fls(A)$ are satisfied as well because no $A^{F_i}$ has been removed.
For ${\mathcal F} = SFLP$, consider a rule $supp(a)$ of the form (\ref{eq:completion}) such that $a \in J$.
Since $I|_P \setminus X \models_s P^{I|_P}$, there is rule $a \leftarrow A$ in $P^{I|_P}$ such that $I|_P \setminus X \models A$.
Hence, by construction of $J$, $A^T \in J$ and thus $J \models supp(a)$.
\end{proof}

\section{Conclusion}

In this paper, we have first defined a new semantics for programs with
generalized atoms, called supportedly stable models, supportedly FLP,
or SFLP semantics. We have motivated its definition by an anomaly that
arises for the FLP semantics in connection with non-convex generalized
atoms. In particular, only unsupported models may in particular cases
inhibit the stability of candidate models. The new definition
overcomes this anomaly and provides a robust semantics for programs
with generalized atoms. We show several properties of this new
semantics, for example it coincides with the FLP semantics (and thus
also the PSP semantics) on convex programs, and thus also on standard
programs. Furthermore, the complexity of reasoning tasks is equal to
the respective tasks using the FLP semantics. We also provide a
characterization of the new semantics by a Clark-inspired completion.

We observe that other interesting semantics, such as the one by
\cite{ferr-2005-lpnmr}, are also affected by the anomaly on
unsupported models.  In particular, the semantics by
\cite{ferr-2005-lpnmr} is presented for programs consisting of
arbitrary set of propositional formulas, and it is based on a reduct
in which false subformulas are replaced by $\bot$.  Answer sets are
then defined as interpretations being subset-minimal models of their
reducts.  For the syntax considered in this paper, when rewriting
generalized atoms to an equivalent formula, the semantics by
\cite{ferr-2005-lpnmr} coincides with FLP, which immediately shows the
anomaly.  In \cite{ferr-2005-lpnmr} there is also a method for
rewriting aggregates, however $COUNT(\{a,b\}) \neq 1$ is not
explicitly supported, but should be rewritten to $\neg (COUNT(\{a,b\})
= 1)$. Doing this, one can observe that for $P_{\ref{program:1}}$,
$P_{\ref{program:2}}$, $P_{\ref{program:3}}$, and $P_{\ref{program:5}}$ the semantics of
\cite{ferr-2005-lpnmr} behaves like SFLP
(cf.\ Table~\ref{table:models}), while for $P_{\ref{program:4}}$ the
semantics of \cite{ferr-2005-lpnmr} additionally has the answer set
$\{a,b\}$, which is not a supported minimal model of the FLP
reduct. $P_{\ref{program:4}}$ therefore shows that the two semantics
do not coincide, even if generalized atoms are interpreted as their negated
complements, and the precise relationship is left
for further study. However, we also believe that rewriting a
generalized atom into its negated complement is not always natural,
and we are also not convinced that there should be a semantic
difference between a generalized atom and its negated complement.

The second part of the paper concerns the question of compactly
compiling generalized atoms away, to arrive at a program that contains
only traditional atoms whose answer sets are in a one-to-one
correspondence with the original program. Previously existing
complexity results indicated that such a translation can exist, but
that it has to make use of disjunction in rule heads. However, no such
method is currently known. We show that similar techniques can
be used for both FLP and the new SFLP semantics when non-convex aggregates 
are represented in disjunctive normal form.

Concerning future work, implementing a reasoner supporting the new
semantics would be of interest. However, we believe that it would
actually be more important to collect example programs that contain
non-convex generalized atoms in recursive definitions. We have
experimented with a few simple domains stemming from game theory (as
outlined in the introduction), but we are not aware of many other
attempts. Our intuition is that such programs would be written in
several domains that describe features with feedback loops, which
applies to many so-called complex systems. Also computing or checking
properties of neural networks might be a possible application in this
area. Another, quite different application area could be systems that
loosely couple OWL ontologies with rule bases, for instance by means
of HEX programs. HEX atoms interfacing to ontologies will in general
not be convex, and therefore using them in recursive definitions falls
into our framework, where the FLP and SFLP semantics differ.

Another area of future work arises from the fact that rules like $a
\leftarrow a$ are not irrelevant for the SFLP semantics. To us, it is
not completely clear whether this is a big drawback. However, we
intend to study variants of the SFLP semantics that do not exhibit
this peculiarity.

\bibliography{bibtex}

\begin{thebibliography}{}

\bibitem[\protect\citeauthoryear{Alviano and
  Faber}{2013}]{alvi-fabe-2013-lpnmr}
Alviano, M., and Faber, W.
\newblock 2013.
\newblock The complexity boundary of answer set programming with generalized
  atoms under the flp semantics.
\newblock In Cabalar, P., and Tran, S.~C., eds., {\em {Logic Programming and
  Nonmonotonic Reasoning --- 12th International Conference (LPNMR 2013)}},
  number 8148 in {Lecture Notes in AI (LNAI)},  67--72.
\newblock Springer Verlag.

\bibitem[\protect\citeauthoryear{Calimeri, Cozza, and
  Ianni}{2007}]{cali-etal-2007-amai}
Calimeri, F.; Cozza, S.; and Ianni, G.
\newblock 2007.
\newblock {External sources of knowledge and value invention in logic
  programming}.
\newblock {\em {Annals of Mathematics and Artificial Intelligence}}
  50(3--4):333--361.

\bibitem[\protect\citeauthoryear{Clark}{1978}]{clar-78}
Clark, K.~L.
\newblock 1978.
\newblock {Negation as Failure}.
\newblock In Gallaire, H., and Minker, J., eds., {\em {Logic and Data Bases}}.
  New York: Plenum Press.
\newblock  293--322.

\bibitem[\protect\citeauthoryear{Dell'Armi \bgroup et al\mbox.\egroup
  }{2003}]{dell-etal-2003a}
Dell'Armi, T.; Faber, W.; Ielpa, G.; Leone, N.; and Pfeifer, G.
\newblock 2003.
\newblock {Aggregate Functions in Disjunctive Logic Programming: Semantics,
  Complexity, and Implementation in DLV}.
\newblock In {\em {Proceedings of the 18th International Joint Conference on
  Artificial Intelligence (IJCAI) 2003}},  847--852.
\newblock Acapulco, Mexico: Morgan Kaufmann Publishers.

\bibitem[\protect\citeauthoryear{Eiter \bgroup et al\mbox.\egroup
  }{2004}]{eite-etal-2004-KR}
Eiter, T.; Lukasiewicz, T.; Schindlauer, R.; and Tompits, H.
\newblock 2004.
\newblock {Combining Answer Set Programming with Description Logics for the
  {S}emantic {W}eb}.
\newblock In {\em Principles of Knowledge Representation and Reasoning:
  Proceedings of the Ninth International Conference (KR2004), Whistler,
  Canada},  141--151.
\newblock Extended Report RR-1843-03-13, Institut f{\"u}r Informationssysteme,
  TU Wien, 2003.

\bibitem[\protect\citeauthoryear{Eiter \bgroup et al\mbox.\egroup
  }{2005}]{eite-etal-2005-ijcai}
Eiter, T.; Ianni, G.; Schindlauer, R.; and Tompits, H.
\newblock 2005.
\newblock {A Uniform Integration of Higher-Order Reasoning and External
  Evaluations in Answer Set Programming}.
\newblock In {\em {International Joint Conference on Artificial Intelligence
  (IJCAI) 2005}},  90--96.

\bibitem[\protect\citeauthoryear{Faber \bgroup et al\mbox.\egroup
  }{2008}]{fabe-etal-2008-tplp}
Faber, W.; Pfeifer, G.; Leone, N.; Dell'Armi, T.; and Ielpa, G.
\newblock 2008.
\newblock Design and implementation of aggregate functions in the dlv system.
\newblock {\em {Theory and Practice of Logic Programming}} 8(5--6):545--580.

\bibitem[\protect\citeauthoryear{Faber, Leone, and
  Pfeifer}{2004}]{fabe-etal-2004-jelia}
Faber, W.; Leone, N.; and Pfeifer, G.
\newblock 2004.
\newblock Recursive aggregates in disjunctive logic programs: Semantics and
  complexity.
\newblock In Alferes, J.~J., and Leite, J., eds., {\em {Proceedings of the 9th
  European Conference on Artificial Intelligence (JELIA 2004)}}, volume 3229 of
  {\em {Lecture Notes in AI (LNAI)}},  200--212.
\newblock Springer Verlag.

\bibitem[\protect\citeauthoryear{Faber, Leone, and
  Pfeifer}{2011}]{fabe-etal-2011-aij}
Faber, W.; Leone, N.; and Pfeifer, G.
\newblock 2011.
\newblock Semantics and complexity of recursive aggregates in answer set
  programming.
\newblock {\em {Artificial Intelligence}} 175(1):278--298.
\newblock Special Issue: John McCarthy's Legacy.

\bibitem[\protect\citeauthoryear{Ferraris}{2005}]{ferr-2005-lpnmr}
Ferraris, P.
\newblock 2005.
\newblock {Answer Sets for Propositional Theories}.
\newblock In Baral, C.; Greco, G.; Leone, N.; and Terracina, G., eds., {\em
  {Logic Programming and Nonmonotonic Reasoning --- 8th International
  Conference, LPNMR'05, Diamante, Italy, September 2005, Proceedings}}, volume
  3662,  119--131.
\newblock Springer Verlag.

\bibitem[\protect\citeauthoryear{Liu and
  Truszczy{\'n}ski}{2006}]{liu-trus-2006-jair}
Liu, L., and Truszczy{\'n}ski, M.
\newblock 2006.
\newblock Properties and applications of programs with monotone and convex
  constraints.
\newblock {\em {Journal of Artificial Intelligence Research}} 27:299--334.

\bibitem[\protect\citeauthoryear{Niemel{\"a} and Simons}{2000}]{niem-simo-2000}
Niemel{\"a}, I., and Simons, P.
\newblock 2000.
\newblock {Extending the {S}models System with Cardinality and Weight
  Constraints}.
\newblock In Minker, J., ed., {\em Logic-Based Artificial Intelligence}.
  Dordrecht: Kluwer Academic Publishers.
\newblock  491--521.

\bibitem[\protect\citeauthoryear{Niemel{\"a}, Simons, and
  Soininen}{1999}]{niem-etal-99}
Niemel{\"a}, I.; Simons, P.; and Soininen, T.
\newblock 1999.
\newblock {Stable Model Semantics of Weight Constraint Rules}.
\newblock In Gelfond, M.; Leone, N.; and Pfeifer, G., eds., {\em {Proceedings
  of the 5th International Conference on Logic Programming and Nonmonotonic
  Reasoning (LPNMR'99)}}, volume 1730 of {\em {Lecture Notes in AI (LNAI)}},
  107--116.
\newblock El Paso, Texas, USA: Springer Verlag.

\bibitem[\protect\citeauthoryear{Pelov, Denecker, and
  Bruynooghe}{2007}]{pelo-etal-2007-tplp}
Pelov, N.; Denecker, M.; and Bruynooghe, M.
\newblock 2007.
\newblock {Well-founded and Stable Semantics of Logic Programs with
  Aggregates}.
\newblock {\em {Theory and Practice of Logic Programming}} 7(3):301--353.

\bibitem[\protect\citeauthoryear{Pelov}{2004}]{pelo-2004}
Pelov, N.
\newblock 2004.
\newblock {\em {Semantics of Logic Programs with Aggregates}}.
\newblock Ph.D. Dissertation, Katholieke Universiteit Leuven, Leuven, Belgium.

\bibitem[\protect\citeauthoryear{Son and Pontelli}{2007}]{son-pont-2007-tplp}
Son, T.~C., and Pontelli, E.
\newblock 2007.
\newblock {A Constructive Semantic Characterization of Aggregates in ASP}.
\newblock {\em {Theory and Practice of Logic Programming}} 7:355--375.

\end{thebibliography}
\bibliographystyle{aaai}
\end{document}